\documentclass[letterpaper]{article} 
\usepackage{aaai25}  
\usepackage{times}  
\usepackage{helvet}  
\usepackage{courier}  
\usepackage[hyphens]{url}  
\usepackage{graphicx} 
\usepackage{natbib}  
\usepackage{caption} 
\usepackage{algorithm}
\usepackage{algorithmic}
\usepackage{cuted}

\usepackage{newfloat}
\usepackage{listings}
\DeclareCaptionStyle{ruled}{labelfont=normalfont,labelsep=colon,strut=off} 
\floatstyle{ruled}
\newfloat{listing}{tb}{lst}{}
\floatname{listing}{Listing}
\usepackage{amsmath, amssymb}
\usepackage{graphicx}
\usepackage{multirow}
\usepackage{booktabs}
\usepackage{tikz}

\usepackage{color,colortbl}
\usepackage{xcolor}
\definecolor{bb}{rgb}{0.5, 0.0, 0.5} 
\definecolor{ss}{rgb}{0, 0, 1}

\newtheorem{proposition}{Proposition}
\newtheorem{proof}{Proof}

\DeclareMathOperator*{\argmax}{argmax}
\title{Adaptive Decision Boundary for Few-Shot Class-Incremental Learning}
\author{
    Linhao Li\textsuperscript{\rm 1}\footnote{Linhao Li and Yongzhang Tan are co-first authors.},
    Yongzhang Tan\textsuperscript{\rm 1}\footnotemark[1],
    Siyuan Yang\textsuperscript{\rm 2},
    Hao Cheng\textsuperscript{\rm 1}\footnote{Corresponding Author.},
    Yongfeng Dong\textsuperscript{\rm 1},
    Liang Yang\textsuperscript{\rm 1}
}
\affiliations {
    \textsuperscript{\rm 1}School of Artificial Intelligence and Data Science, Hebei University of Technology, China\\
    \textsuperscript{\rm 2}College of Computing and Data Science, Nanyang Technological University, Singapore\\
    \{lilinhao, chenghao, dongyf\}@hebut.edu.cn, siyuan.yang@ntu.edu.sg, 202232805041@stu.hebut.edu.cn, yangliang@vip.qq.com

}
\usepackage{bibentry}

\begin{document}

\maketitle

\begin{abstract}
Few-Shot Class-Incremental Learning (FSCIL) aims to continuously learn new classes from a limited set of training samples without forgetting knowledge of previously learned classes.
Conventional FSCIL methods typically build a robust feature extractor during the base training session with abundant training samples and subsequently freeze this extractor, only fine-tuning the classifier in subsequent incremental phases.
However, current strategies primarily focus on preventing catastrophic forgetting, considering only the relationship between novel and base classes, without paying attention to the specific decision spaces of each class.
To address this challenge, we propose a plug-and-play \textit{Adaptive Decision Boundary Strategy} (ADBS), which is compatible with most FSCIL methods.
Specifically, we assign a specific decision boundary to each class and adaptively adjust these boundaries during training to optimally refine the decision spaces for the classes in each session.
Furthermore, to amplify the distinctiveness between classes, we employ a novel inter-class constraint loss that optimizes the decision boundaries and prototypes for each class.
Extensive experiments on three benchmarks, namely CIFAR100, miniImageNet, and CUB200, demonstrate that incorporating our ADBS method with existing FSCIL techniques significantly improves performance, achieving overall state-of-the-art results.
\end{abstract}

%
\begin{links}
    \link{Code}{https://github.com/Yongzhang-Tan/ADBS}
\end{links}

\begin{figure}[!ht]
\centering
\includegraphics[width=\linewidth]{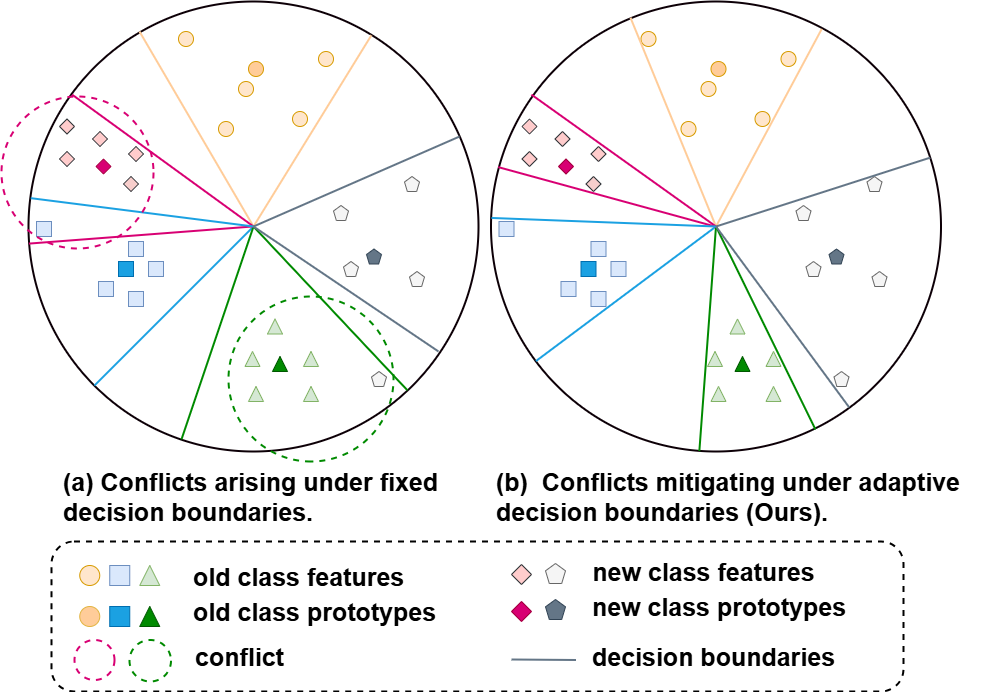}
\caption{
Illustration of FSCIL classification with (a) Fixed and (b) Our proposed adaptive decision boundaries. A fixed decision boundary strategy often struggles to reserve adequate space for new classes at each incremental stage, resulting in space conflicts between old and new classes. In contrast, our proposed adaptive decision boundary strategy can effectively alleviate this issue by adjusting the decision boundaries of both old and new classes.
}
\label{fig:conflict}
\end{figure}

\section{Introduction}

To date, deep Convolutional Neural Networks (CNNs) have achieved significant advancements in the field of computer vision, primarily using models trained on static and pre-collected large-scale annotated datasets~\cite{deng2009imagenet,he2016deep}. 
However, the practical challenge of gradually integrating data from new classes into a model initially trained on existing classes introduces significant obstacles, known as Class Incremental Learning (CIL)~\cite{hou2019learning,li2017learning,rebuffi2017icarl,cheng2025stsp}.
Unlike standard classification tasks, CIL requires handling new classes and restricted access to prior task data in incremental sessions.
Consequently, simply updating the model with new class data may lead to overfitting on these new classes and cause significant performance drops on previously learned classes, a phenomenon known as catastrophic forgetting~\cite{mccloskey1989catastrophic, masana2022class}.
Nonetheless, CIL requires sufficient training data from novel classes. In many applications, data collection and labeling can be prohibitively expensive, posing challenges for implementing CIL in real-world scenarios.
To address these challenges, Few-Shot Class-Incremental Learning (FSCIL) is proposed to address the dual challenges of learning new classes from limited examples and integrating this knowledge into an existing model without forgetting previously acquired information~\cite{tao2020few}.
Critically, the model must maintain a balance between stability and plasticity, preserving previously acquired knowledge while seamlessly integrating new information.

Recent advancements in FSCIL methods~\cite{zhang2021few,cheraghian2021semantic,dong2021few,zhao2021mgsvf,zhou2022forward,chen2022multi,yang2023neural,fan2024dynamic} have shown remarkable performance in image classification tasks.
These methods typically employ a learning pipeline that includes pre-training the model during a data-rich base session, followed by merely constructing a classifier for novel classes on the frozen backbone during incremental sessions.
While this approach effectively integrates new classes while retaining previously learned information, it still remains a significant challenge, \textit{i.e.}, conflicts in the decision space between old and novel classes, as depicted in Figure~\ref{fig:conflict}(a).
When the features of base classes are similar to those of novel classes, classifying novel classes may disrupt the classification of the pre-established base classes.
Moreover, in scenarios with a fixed backbone, base classes typically do not reserve space for unknown novel classes, leading to performance degradation.
Several approaches~\cite{song2023learning,zhou2022forward} have attempted to address this issue by introducing virtual classes to reserve feature space for novel classes. 
However, it remains challenging to predict the feature space of novel classes that significantly deviate from base classes using only data augmentation-based strategies, \textit{e.g.}, mixup.
Additionally, finding a balance between the space reserved for base and novel classes to avoid compromising base classification tasks remains a significant issue.
Previous studies~\cite{peng2022few,zou2022margin,guo2023decision} have attempted to introduce decision boundaries in FSCIL to improve inter-class cohesion. 
However, these methods, which rely on predetermined hyperparameters to set boundaries, often fail to accurately define precise demarcations. 
The aforementioned observations prompt us to construct adaptive decision boundaries to accurately predict base classes and reserve more space for incoming novel classes.

In this paper, we propose a novel \textit{Adaptive Decision Boundary Strategy} (ADBS) for FSCIL, which assigns a specific decision boundary to each class that dynamically adjusts based on the training data.
To enhance the distinctiveness between classes, we also implement a novel inter-class constraint loss that optimizes the decision boundaries and prototypes for each class. 
The proposed ADBS significantly improves class separation and effectively reduces conflicts in the feature space between existing and incoming classes, as shown in Figure~\ref{fig:conflict}(b).
Furthermore, ADBS is a plug-and-play module that can be easily integrated with existing FSCIL frameworks without necessitating modifications to the network architecture.

To summarize, our main contributions are as follows:
\begin{itemize}
\item We introduce a plug-and-play \textit{Adaptive Decision Boundary Strategy} (ADBS) designed to mitigate conflicts between base and novel classes in the feature space within FSCIL tasks.
Our theoretical analysis verifies that this strategy effectively differentiates class centers and optimizes their boundaries.
\item We employ an inter-class constraint to optimize the decision boundaries and prototypes for each class, further enhancing the distinguishability between classes.
\item We evaluate the ABDS method over three FSCIL benchmarks: CIFAR100, \textit{mini}ImageNet, and CUB200.
Experimental results and visualizations demonstrate that incorporating ADBS with existing FSCIL algorithms, including both the baseline and state-of-the-art methods, consistently enhances performance and achieves state-of-the-art performance.
\end{itemize}

\section{Related Work}

\subsection{Few-Shot Class-Incremental Learning}

The concept of Few-Shot Class-Incremental Learning (FSCIL) was first introduced in TOPIC~\cite{tao2020few}, which aims to address the dual challenges of catastrophic forgetting and learning new classes incrementally from a limited number of labeled samples. 
TOPIC tackled these issues by implementing a \textit{neural gas} (NG) network. 
Current methods in FSCIL can be categorized into two main groups: one updates the backbone network during incremental sessions \cite{cheraghian2021semantic,dong2021few,kang2022soft,li2024analogical}, while the others maintain a fixed backbone. 
The latter group attempts to suppress forgetting old knowledge while adapting smoothly to novel classes using various approaches.
Several methods improve model performance by constructing diverse powerful classifiers, such as ETF~\cite{yang2023neural} or stochastic classifiers \cite{kalla2022s3c}.
Other methods~\cite{zhou2022forward,song2023learning} introduce virtual classes to pre-allocate feature space for novel classes, while additional methods include episodic training~\cite{zhou2022few,zhu2021self} or ensemble learning~\cite{ji2023memorizing} to boost the capabilities of backbone.
Recent methods~\cite{yang2021free,liu2023learnable} consider to employ distribution calibration to adjust the classifier.
OrCo~\cite{ahmed2024orco} promotes class separation by leveraging feature orthogonality in the representation space and contrastive learning.

\subsection{Boundary-based Method}

Boundary-based methods, which focus on learning optimal decision boundaries, are extensively employed in various vision tasks such as image classification~\cite{chen2020boundary,wang2024integrated}, semantic segmentation~\cite{liu2022devil}, and domain generalization~\cite{dayal2024madg}. 
The widespread use and success of these methods in diverse applications highlight their efficacy, showcasing their capability to manage complex visual data with exceptional accuracy and reliability.

Several methods also explore boundary strategies within the FSCIL domain.
In prior studies, ALICE~\cite{peng2022few} utilizes angular penalty loss with hyperparameter-defined boundaries to enhance inter-class cohesion, yet these boundaries are not effectively utilized during inference. Conversely, CLOM~\cite{zou2022margin} employs hyperparameter-based methods to establish positive and negative boundaries by considering distances between class prototypes, while DBONet~\cite{guo2023decision} assumes that data feature vectors follow a spherical Gaussian distribution and employs within-class variance to define boundaries. 
Although these approaches promote class separation to some degree, they often struggle to achieve precise boundary accuracy. 
Additionally, traditional boundary-based methods frequently require modifications to the model architecture, which restricts their widespread application. 
To overcome these limitations, we introduce a plug-and-play \textit{Adaptive Decision Boundary Strategy} that captures more accurate boundaries and effectively resolves conflicts between new and old classes in the feature space.

\section{Preliminary}

\subsection{Problem Statement}

In FSCIL, we conduct a series of continuous learning sessions, each featuring a steady stream of training data represented as \(D_{\text{train}} \!=\! \{D^t_{\text{train}}\}_{t=0}^T\).
Each subset \(D^t_{\text{train}} = \{(x_i, y_i)\}_{i=0}^{N_t}\) contains training samples from session \(t\), with \(x_i\) and \(y_i\) denote the \(i\)-th image and its corresponding label, respectively. 
The initial session, termed the base session, provides a substantial amount of training data.
Each subsequent session, referred to as an incremental session, adopts an \(N\)-way, \(K\)-shot setting, which includes \(N\) classes, each with \(K\) samples.
The label space for the $t$-th session is  denoted by $\mathcal{C}^{t}$, which is disjoint between different sessions, \textit{i.e.}, $\mathcal{C}^{t_1} \cap \mathcal{C}^{t_2} = \varnothing$ when  $t_1 \neq t_2$.
The model trained on \(D^t_{\text{train}}\) should be evaluated on \(D^t_{\text{test}}\), which includes all classes encountered up to the \(t\)-th session, represented as \(\bigcup_{i=0}^t \mathcal{C}^{i} \).

\subsection{Base Pretraining and Novel Fine-tuning Strategy}

The Base Classes Pretraining and Novel Classes Fine-tuning (BPNF) strategy \cite{tian2024survey} is a common approach in the FSCIL scenario, which involves initial pre-training on abundant base data, followed by fine-tuning to enhance adaptation to novel classes during the incremental phase.

In general, an FSCIL model is decomposed into two components: a feature extractor $f(\cdot)$ and a dynamic classifier with weights $W$.
The output of the model is represented as:
\begin{equation}
    \phi (x) = W^{\top} f(x),
\label{eq:decouple}
\end{equation}
where \(\phi (x) \in \mathbb{R}^{  \lvert \mathcal{C}\lvert \times 1}\), \(W \in \mathbb{R}^{d \times \lvert \mathcal{C} \lvert}\), and \(f(x) \in \mathbb{R}^{d \times 1}\) with $d$ and $\mathcal{C}$ denotes the output feature dimension and number of classes, respectively.

Specifically, BPNF first leverages sufficient data in the base session to train the model by optimizing the loss for each sample $x$:
\begin{equation}
    \mathcal{L}_{cls}(\phi; x, y) = l(\phi(x), y),
\label{eq:cls}
\end{equation}
where \(l(\cdot, \cdot)\) denotes the cross-entropy loss function.

After the base training phase, $f(\cdot)$ is fixed, and the classifier is expanded in each subsequent incremental session:
$W \!=\! \bigcup_{i=0}^t \{w_1^i, w_2^i, \cdots, w_{\lvert \mathcal{C}^{i} \lvert }^i\}$, where each term is parameterized by the prototype of the corresponding novel class:
\begin{equation}
    w_c^t = \frac{1}{K} \sum_{i=1}^{K} f(x_i).
    \label{eq:prototype}
\end{equation}

In $t$-th session, inference is performed using the Nearest Class Mean (NCM) \cite{mensink2013distance} algorithm to evaluate the accuracy of all encountered classes by predicting the class $\hat{c}_x^t$
with:
\begin{equation}
    \hat{c}_x^t 
    = \argmax_{c,t} \text{ sim } (f(x), w_c^t),
\label{eq:ncm}
\end{equation}
where \( \text{sim}(\mathbf{x}, \mathbf{y}) = \frac{\mathbf{x}^\top \mathbf{y}}{\|\mathbf{x}\| \|\mathbf{y}\|}
 \) represents the cosine similarity between two vectors.

\section{Methodology}

In this section, we provide a detailed description of our proposed methodologies, \textit{Adaptive Decision Boundary} and \textit{Inter-class Constraint}.
The overview of the entire training pipeline is illustrated in Figure~\ref{fig:pipeline}.

\begin{figure*}[!ht]
    \centering
    \includegraphics[width=\linewidth]{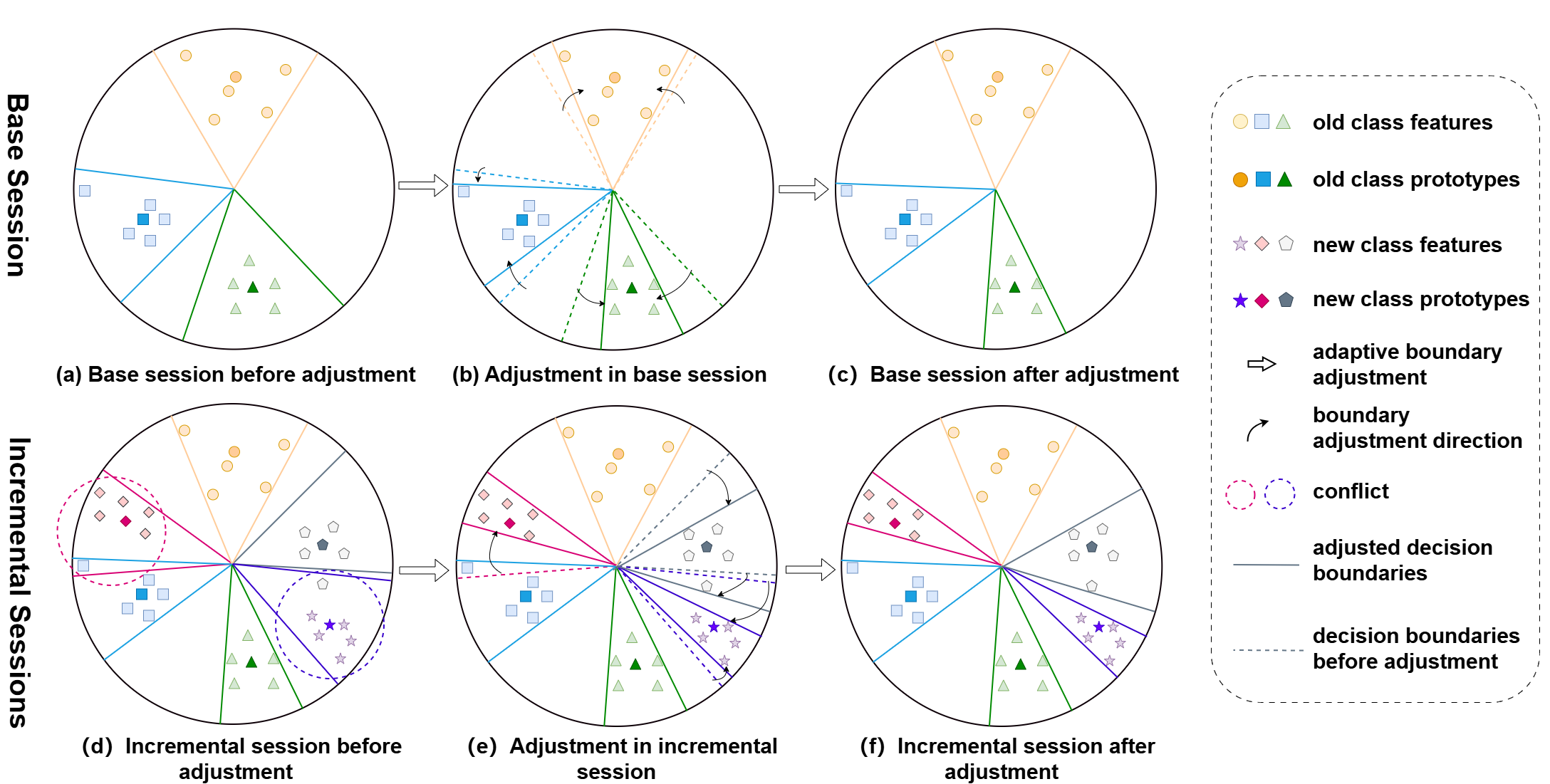}
    \caption{The overall pipeline of our \textit{Adaptive Decision Boundary Strategy} (ADBS).
    In the base session, we compress the decision boundaries of the base classes to reserve feature space for the upcoming new classes, as depicted in (a)-(c).
    Subsequently, in the incremental sessions, while maintaining the boundaries of the base classes, we dynamically adjust the boundaries of the new classes to optimize classification performance and compress the boundaries of current classes to allocate feature space for forthcoming new classes, as shown in (d)-(f).
    Furthermore, we impose Inter-class Constraint (IC) to enhance class distinguishability in each session.
    }
    \label{fig:pipeline}
\end{figure*}

\subsection{Adaptive Decision Boundary}

To more effectively partition the feature space, we propose the \textit{Adaptive Decision Boundary} (ADB). 
This approach involves assigning a unique decision boundary to each class and continuously adapting it throughout the training process.
This strategy conserves the feature space utilized by the base classes, thus liberating additional feature space for the coming new classes.

Previous research typically assigns a unified decision boundary to all classes instead of implementing individual boundaries for each class. 
This common practice usually results in a decision boundary that is determined by the class with the most dispersed inter-class features, which can lead to an excessively large boundary for all classes, as illustrated in Figure~\ref{fig:pipeline} (a). 
This practice can clutter the feature space for new classes, resulting in conflicts when new classes are introduced.
Drawing inspiration from~\cite{zou2022margin}, we establish a specific boundary for each class and further propose adaptive adjustments to these boundaries.

In the base session, we initially assign a decision boundary to each base class, formally defined as: \( M = \{ m_1^0, m_2^0, \dots, m_{ \lvert \mathcal{C}^0 \lvert }^0 \} \), \( M \in \mathbb{R}^{1 \times \lvert \mathcal{C}^0 \lvert } \). We then incorporate these adaptive boundaries into the model as follows:
\begin{equation} 
    \phi (x) = (W \cdot M)^{\top} f(x).
\end{equation}

Throughout the training process, we apply Eq.~\ref{eq:cls} to refine both the original model and the boundaries \(M\), thereby facilitating the adaptive adjustment of \(M\) towards the optimal decision boundary.

Compared to the less precise boundaries determined by hyperparameters in CLOM~\cite{zou2022margin}, our method adaptively adjusts the decision boundaries during training.
Assuming that clear distinctions already exist between base classes without the use of boundaries \(M\), our method allows the boundaries of classes with smaller feature spaces to adaptively and substantially contract, guided by the loss function described in Eq.~\ref{eq:cls}. 
Concurrently, the boundaries of other classes will also contract to a degree, as illustrated in Figure \ref{fig:pipeline} (b). 
This strategy liberates additional space for new classes, effectively reducing the conflicts between existing and incoming classes within the feature space.

In incremental sessions, previous methods typically employ only Eq.~\ref{eq:prototype} to compute prototypes without adjusting boundaries.
As a result of continuing to use the static boundaries of the old classes, the new classes fail to adapt to their own distribution, leading to suboptimal classification performance and catastrophic forgetting, as illustrated in Figure~\ref{fig:pipeline} (d).
Therefore, we resolve to continuously fine-tune the boundaries for new classes during these sessions.

In each incremental session \(t\), following the BPNF strategy, we first update the classifier using prototypes of the new classes.
Next, we adjust the decision boundaries while keeping the backbone network and the classifier fixed. 
To accomplish this, we first expand \(M\) to  \( M = \{ m_1^0, m_2^0, \cdots, m_{ \lvert \mathcal{C}^0 \lvert }^0 \} \cup \cdots \cup \{ m_1^t, m_2^t, \cdots, m_{ \lvert \mathcal{C}^t \lvert }^t \}  \).
Subsequently, we calculate the mean boundary of all old classes to set the initial boundaries for the new classes in the current session. 
Concurrently, the boundaries of all old classes remain unchanged, while the boundaries for new classes are specifically adjusted using the Eq.~\ref{eq:cls}.

By fine-tuning the boundaries of new classes, we are able to improve the classification performance of these classes while preserving the knowledge of the old classes, as shown in Figure~\ref{fig:pipeline} (f).
During inference, we utilize the previously trained boundaries. 
Eq.~\ref{eq:ncm} is reformulated as follows:
\begin{equation}
    \hat{c}_x^t = \argmax_{c,t} m_c^t \text{ sim } (f(x), w_c^t).
\end{equation}

\begin{table*}[ht]
    \centering
    \small
    \setlength{\tabcolsep}{1mm}
    \begin{tabular}
    {lcccccccccccc}
        \toprule
        \multirow{2}{*}{Methods} & \multicolumn{9}{c}{Accuracy in each session (\%)} & \multirow{2}{*}{Average} 
        \\
        \cmidrule(lr){2-10} 
        & 0 & 1 & 2 & 3 & 4 & 5 & 6 & 7 & 8 & Acc. 
        \\
        \midrule
Topic {\cite{tao2020few}} & {64.10} & {55.88} & {47.07} & {45.16} & {40.11} & {36.38} & {33.96} & {31.55} & {29.37} & {42.62 }  \\
CEC {\cite{zhang2021few}} & {73.07} & {68.88} & {65.26} & {61.19} & {58.09} & {55.57} & {53.22} & {51.34} & {49.14} & {59.53 }  \\
FACT {\cite{zhou2022forward}} & {74.60} & {72.09} & {67.56} & {63.52} & {61.38} & {58.36} & {56.28} & {54.24} & {52.10} & {62.24 }  \\
C-FSCIL {\cite{hersche2022constrained}} & {77.47} & {72.40} & {67.47} & {63.25} & {59.84} & {56.95} & {54.42} & {52.47} & {50.47} & {61.64 }  \\
CLOM† {\cite{zou2022margin}} & {74.20} & {69.83} & {66.17} & {62.39} & {59.26} & {56.48} & {54.36} & {52.16} & {50.25} & {60.57 }  \\
DBONet† {\cite{guo2023decision}} & {77.81} & {73.62} & {71.04} & {66.29} & {63.52} & {61.01} & {58.37} & {56.89} & {55.78} & {64.93 }  \\
MCNet {\cite{ji2023memorizing}} & {73.30} & {69.34} & {65.72} & {61.70} & {58.75} & {56.44} & {54.59} & {53.01} & {50.72} & {60.40 }  \\
SoftNet {\cite{kang2022soft}} & {72.62} & {67.31} & {63.05} & {59.39} & {56.00} & {53.23} & {51.06} & {48.83} & {46.63} & {57.57 }  \\
WaRP {\cite{kim2023warping}} & {80.31} & {75.86} & {71.87} & {67.58} & {64.39} & {61.34} & {59.15} & {57.10} & {54.74} & {65.82 }  \\
TEEN {\cite{wang2024few}} & {74.92} & {72.65} & {68.74} & {65.01} & {62.01} & {59.29} & {57.90} & {54.76} & {52.64} & {63.10 }  \\
NC-FSCIL {\cite{yang2023neural}} & {82.52} & {76.82} & {73.34} & {69.68} & {66.19} & {62.85} & {60.96} & {59.02} & {56.11} & {67.50 }  \\
ALFSCIL {\cite{li2024analogical}} & {80.75} & {77.88} & {72.94} & {68.79} & {65.33} & {62.15} & {60.02} & {57.68} & {55.17} & {66.75 }  \\
DyCR {\cite{pan2024dycr}} & {75.73} & {73.29} & {68.71} & {64.80} & {62.11} & {59.25} & {56.70} & {54.56} & {52.24} & {63.04 }  \\
\midrule
baseline & {73.92} & {67.14} & {63.71} & {60.07} & {57.10} & {54.85} & {52.52} & {50.49} & {48.60} & {58.71 }  \\
baseline+ADBS & \textbf{79.93} & \textbf{75.22} & \textbf{71.11} & \textbf{65.99} & \textbf{62.46} & \textbf{58.38} & \textbf{55.96} & \textbf{53.72} & \textbf{51.15} & \textbf{63.77 }  \\
  & \textbf{(+6.02)} & \textbf{(+8.08)} & \textbf{(+7.40)} & \textbf{(+5.92)} & \textbf{(+5.36)} & \textbf{(+3.53)} & \textbf{(+3.43)} & \textbf{(+3.22)} & \textbf{(+2.55)} & \textbf{(+5.06)  } \\
\midrule
OrCo* {\cite{ahmed2024orco}} & {79.77} & {63.29} & {62.39} & {60.13} & {58.76} & {56.56} & {55.49} & {54.19} & {51.12} & {60.19 }  \\
OrCo+ADBS & {79.77} & \textbf{63.46} & {61.89} & \textbf{60.43} & \textbf{59.23} & {56.32} & \textbf{55.76} & \textbf{54.48} & \textbf{51.54} 
& \textbf{60.32 }  \\
  & {(+0.00)} & \textbf{(+0.17)} & {(-0.50)} & \textbf{(+0.29)} & \textbf{(+0.46)} & {(-0.25)} & \textbf{(+0.27)} & \textbf{(+0.30)} & \textbf{(+0.42)} & \textbf{(+0.13)  } \\
\midrule
ALICE†* {\cite{peng2022few}} & {80.37} & {72.34} & {67.67} & {63.61} & {61.11} & {58.53} & {57.40} & {55.43} & {53.46} & {63.32 }  \\
ALICE+ADBS & {80.12} & \textbf{74.11} & \textbf{70.51} & \textbf{66.72} & \textbf{63.90} & \textbf{61.25} & \textbf{60.00} & \textbf{58.07} & \textbf{56.00} & \textbf{65.63 }  \\
  & {(-0.25)} & \textbf{(+1.77)} & \textbf{(+2.84)} & \textbf{(+3.11)} & \textbf{(+2.79)} & \textbf{(+2.72)} & \textbf{(+2.60)} & \textbf{(+2.64)} & \textbf{(+2.54)} & \textbf{(+2.31)  } \\
\midrule
SAVC* {\cite{song2023learning}} & {78.60} & {72.95} & {68.73} & {64.59} & {61.41} & {58.46} & {56.29} & {54.40} & {52.19} & {63.07 }  \\
SAVC+ADBS & \textbf{85.13} & \textbf{80.39} & \textbf{77.07} & \textbf{72.61} & \textbf{69.54} & \textbf{66.54} & \textbf{64.70} & \textbf{62.72} & \textbf{60.60} 
& \textbf{71.03 }  \\
  & \textbf{(+6.53)} & \textbf{(+7.43)} & \textbf{(+8.34)} & \textbf{(+8.03)} & \textbf{(+8.13)} & \textbf{(+8.08)} & \textbf{(+8.41)} & \textbf{(+8.32)} & \textbf{(+8.41)} & \textbf{(+7.96)  } \\
        \bottomrule
    \end{tabular}
    \caption{Comparison with SOTA methods on CIFAR100 for FSCIL. †: Boundary-based method. *: Reproduced results.}
    \label{tab:cifar100}
\end{table*}

\subsection{Inter-class Constraint}

With the help of adaptive decision boundaries, we can adjust the decision space of existing classes.
To further enhance the distinguishability between classes, the decision boundary $M$ should satisfy the following proposition:
\begin{proposition}
\label{propos1}
Given a classifier with weights $W$, the adaptive boundary strategy could help to better separate classes $i,j$ if the boundary weights $m_i$ and $m_j$ satisfy the following equation:
\begin{equation}
    \quad (1 - m_i)p_i^\top w_i + (m_j - 1)p_i^\top w_j \leq 0, \quad \forall i \neq j, 
    \label{eq:constraint}
\end{equation}
where $p_i$ denotes the prototype of class $i$. $w_i$ and $w_j$ correspond to the weights of classes $i$ and $j$ in the classifier, respectively.
\begin{proof}
Please refer to the \textbf{supplemental materials}.
\end{proof}
\end{proposition}

Building upon Proposition~\ref{propos1}, we introduce a novel Inter-class Constraint (IC) loss to optimize the boundaries $M$ as:
\begin{equation}
    \mathcal{L}_{IC} = \sum_{i=1}^{N}
    \sum_{j=1}^{N}
    \max(0, (1 - m_i)p_i^\top w_i + (m_j - 1)p_i^\top w_j),
\end{equation}
where $N$ denotes the total number of classes and $p_i$ is the normalized prototype of $i$-th class.

IC loss helps optimize the decision boundaries and prototypes for each class, facilitating further separation of the classes and the establishment of clearer decision boundaries.

The overall objective function is defined as:
\begin{equation}
    \mathcal{L} = \mathcal{L}_{cls} + \alpha \mathcal{L}_{IC},
    \label{eq:total_loss}
\end{equation}
where $\alpha$ denotes the weighting parameters of $\mathcal{L}_{IC}$.

\begin{table*}[ht]
    \centering
    \small
    \setlength{\tabcolsep}{1mm}
    \begin{tabular}{lccccccccccccc}
        \toprule
        \multirow{2}{*}{Methods} & \multicolumn{9}{c}{Accuracy in each session (\%)} & \multirow{2}{*}{Average} 
        \\
        \cmidrule(lr){2-10} 
        & 0 & 1 & 2 & 3 & 4 & 5 & 6 & 7 & 8 & Acc. 
        \\
        \midrule
Topic {\cite{tao2020few}} & {61.31} & {50.09} & {45.17} & {41.13} & {37.48} & {35.52} & {32.19} & {29.46} & {24.42} & {39.64 }  \\
CEC {\cite{zhang2021few}} & {72.00} & {66.83} & {62.97} & {59.43} & {56.70} & {53.73} & {51.19} & {49.24} & {47.63} & {57.75 }  \\
FACT {\cite{zhou2022forward}} & {72.56} & {69.63} & {66.38} & {62.77} & {60.60} & {57.33} & {54.34} & {52.16} & {50.49} & {60.70 }  \\
C-FSCIL {\cite{hersche2022constrained}} & {76.40} & {71.14} & {66.46} & {63.29} & {60.42} & {57.46} & {54.78} & {53.11} & {51.41} & {61.61 }  \\
CLOM† {\cite{zou2022margin}} & {73.08} & {68.09} & {64.16} & {60.41} & {57.41} & {54.29} & {51.54} & {49.37} & {48.00} & {58.48 }  \\
DBONet† {\cite{guo2023decision}} & {74.53} & {71.55} & {68.57} & {65.72} & {63.08} & {60.64} & {57.83} & {55.21} & {53.82} & {63.44 }  \\
MCNet {\cite{ji2023memorizing}} & {72.33} & {67.70} & {63.50} & {60.34} & {57.59} & {54.70} & {52.13} & {50.41} & {49.08} & {58.64 }  \\
SoftNet {\cite{kang2022soft}} & {77.17} & {70.32} & {66.15} & {62.55} & {59.48} & {56.46} & {53.71} & {51.68} & {50.24} & {60.86 }  \\
WaRP {\cite{kim2023warping}} & {72.99} & {68.10} & {64.31} & {61.30} & {58.64} & {56.08} & {53.40} & {51.72} & {50.65} & {59.69 }  \\
TEEN {\cite{wang2024few}} & {73.53} & {70.55} & {66.37} & {63.23} & {60.53} & {57.95} & {55.24} & {53.44} & {52.08} & {61.44 }  \\
NC-FSCIL {\cite{yang2023neural}} & {84.02} & {76.80} & {72.00} & {67.83} & {66.35} & {64.04} & {61.46} & {59.54} & {58.31} & {67.82 }  \\
ALFSCIL {\cite{li2024analogical}} & {81.27} & {75.97} & {70.97} & {66.53} & {63.46} & {59.95} & {56.93} & {54.81} & {53.31} & {64.80 }  \\
DyCR {\cite{pan2024dycr}} & {73.18} & {70.16} & {66.87} & {63.43} & {61.18} & {58.79} & {55.00} & {52.87} & {51.08} & {61.40 }  \\
\midrule
baseline & {74.80} & {68.58} & {64.69} & {61.71} & {59.06} & {56.26} & {54.09} & {52.30} & {50.81} & {60.25 }  \\
baseline+ADBS & \textbf{81.40} & \textbf{75.03} & \textbf{71.03} & \textbf{68.00} & \textbf{65.56} & \textbf{61.87} & \textbf{59.04} & \textbf{56.87} & \textbf{55.38} & \textbf{66.02 }  \\
  & \textbf{(+6.60)} & \textbf{(+6.45)} & \textbf{(+6.34)} & \textbf{(+6.29)} & \textbf{(+6.50)} & \textbf{(+5.61)} & \textbf{(+4.95)} & \textbf{(+4.58)} & \textbf{(+4.57)} & \textbf{(+5.77)  } \\
\midrule
OrCo* {\cite{ahmed2024orco}} & {83.22} & {74.60} & {71.89} & {67.65} & {65.53} & {62.73} & {60.33} & {58.51} & {57.62} & {66.90 }  \\
OrCo+ADBS & \textbf{84.30} & \textbf{78.02} & \textbf{73.96} & \textbf{70.07} & \textbf{67.47} & \textbf{64.65} & \textbf{61.74} & \textbf{59.92} & \textbf{58.42} 
& \textbf{68.73 }  \\
  & \textbf{(+1.08)} & \textbf{(+3.42)} & \textbf{(+2.07)} & \textbf{(+2.41)} & \textbf{(+1.95)} & \textbf{(+1.92)} & \textbf{(+1.41)} & \textbf{(+1.41)} & \textbf{(+0.80)} & \textbf{(+1.83)  } \\
\midrule
ALICE†* {\cite{peng2022few}} & {84.23} & {72.08} & {68.77} & {65.57} & {63.21} & {60.58} & {58.26} & {56.69} & {55.76} & {65.02 }  \\
ALICE+ADBS & {83.92} & \textbf{73.11} & \textbf{69.57} & \textbf{66.03} & \textbf{63.82} & \textbf{61.26} & \textbf{58.64} & \textbf{57.07} & \textbf{56.39} & \textbf{65.53 }  \\
  & {(-0.32)} & \textbf{(+1.03)} & \textbf{(+0.80)} & \textbf{(+0.46)} & \textbf{(+0.61)} & \textbf{(+0.68)} & \textbf{(+0.38)} & \textbf{(+0.38)} & \textbf{(+0.63)} & \textbf{(+0.52)  } \\
\midrule
SAVC* {\cite{song2023learning}} & {81.38} & {76.12} & {71.79} & {68.21} & {65.34} & {61.94} & {59.13} & {57.06} & {55.60} & {66.29 }  \\
SAVC+ADBS & \textbf{87.10} & \textbf{81.11} & \textbf{77.16} & \textbf{74.19} & \textbf{72.05} & \textbf{68.34} & \textbf{65.14} & \textbf{63.39} & \textbf{62.72} 
& \textbf{72.36 }  \\
  & \textbf{(+5.72)} & \textbf{(+4.98)} & \textbf{(+5.37)} & \textbf{(+5.97)} & \textbf{(+6.71)} & \textbf{(+6.40)} & \textbf{(+6.01)} & \textbf{(+6.33)} & \textbf{(+7.12)} & \textbf{(+6.07)  } \\
        \bottomrule
    \end{tabular}
    \caption{Comparison with SOTA methods on miniImageNet for FSCIL. †: Boundary-based method. *: Reproduced results.}
    \label{tab:miniimagenet}
\end{table*}

\begin{figure*}[!htbp]
    \centering
    \includegraphics[width=0.9\linewidth]{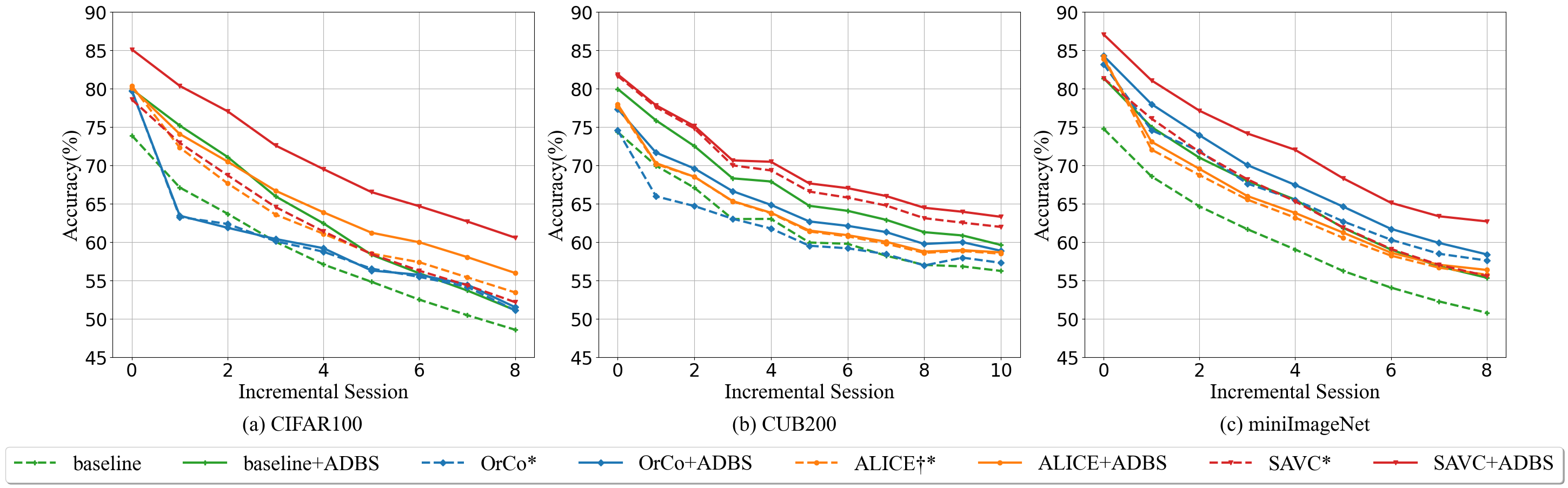}
    \caption{Comparison with different baseline methods on CIFAR100, CUB200, and miniImageNet. The dashed line represents the accuracy of the baseline method, while the solid line depicts the classification performance with our proposed ADBS.}
    \label{fig:baseline_comparison}
\end{figure*}

\begin{table*}[!htbp]
    \centering
     \begin{tabular}{cccccccccccccc}
        \toprule
        \multirow{2}{*}{\textbf{ADB}} & \multirow{2}{*}{\textbf{IC}} & \multicolumn{9}{c}{Accuracy in each session (\%)} & \multirow{2}{*}{$\Delta_{last}$} 
        \\
        \cmidrule(lr){3-11} 
        & & 0 & 1 & 2 & 3 & 4 & 5 & 6 & 7 & 8 &  \\
        \midrule
                    &            & {74.80} & {68.58} & {64.69} & {61.71} & {59.06} & {56.26} & {54.09} & {52.30} & {50.81} & {-}  \\
        \checkmark  &            &{80.13} & {74.74} & {70.21} & {66.51} & {63.67} & {60.62} & {57.92} & {55.85} & {54.67} & {+3.86}  \\
        \checkmark  & \checkmark & \textbf{{81.40}} & \textbf{{75.03}} & \textbf{{71.03}} & \textbf{{68.00}} & \textbf{{65.56}} & \textbf{{61.87}} & \textbf{{59.04}} & \textbf{{56.87}} & \textbf{{55.38}} & \textbf{{+4.57}}  \\
        \bottomrule
    \end{tabular}
    \caption{Ablation studies on miniImageNet benchmark. \textbf{ADB} and \textbf{IC} denote \textit{Adaptive Decision Boundary} and \textit{Inter-class Constraint}, respectively.
    \(\Delta_{last}\): Relative improvements of the last sessions compared to the fixed decision boundary baseline.}
    \label{tab:ablation_miniimagenet}
\end{table*}

\section{Experiment}

\subsection{Experimental Setup}

\textbf{Dataset} Following the benchmark setting~\cite{tao2020few}, we evaluate the effectiveness of our proposed ADBS method on three datasets, \textit{i.e.}, CIFAR100~\cite{krizhevsky2009learning}, miniImageNet~\cite{deng2009imagenet}, and Caltech-UCSD Birds-200-2011 (CUB200)~\cite{wah2011caltech}.

\noindent\textbf{Implementation Details} We integrate our approach into four FSCIL methods: the standard FSCIL baseline, SAVC \cite{song2023learning}, ALICE \cite{peng2022few}, and OrCo \cite{ahmed2024orco}.
We implemente these methods using the official codes released by the authors to ensure a fair comparison.~\footnote{More details about datasets and experimental settings are included in the supplementary materials.}

\noindent\textbf{Evaluation Protocol and Metric} Following previous FSCIL studies~\cite{tao2020few,zhang2021few}, we report the Top-1 accuracy of current classes after the base session and each incremental session.

\subsection{Main Results}

We conduct experimental comparisons on three benchmark datasets~\footnote{The result table of CUB200 dataset can be found in the supplementary materials.}, \textit{i.e.}, CIFAR100, miniImageNet, and CUB200, shown in Tables.~\ref{tab:cifar100},~\ref{tab:miniimagenet}, and Figure~\ref{fig:baseline_comparison}.

Figure~\ref{fig:baseline_comparison} shows that our proposed ADBS consistently enhances the performance of all integrated FSCIL methods across all datasets. 
The results further reveal that integrating SAVC with ADBS (SAVC+ADBS) achieves state-of-the-art performance on all datasets, thereby confirming the effectiveness of the proposed strategy.
Furthermore, we observed that for baseline methods and SAVC, which do not optimize for boundaries, our proposed ADBS continuously compresses the boundaries in each session to allocate feature space, effectively alleviating conflicts between new and old classes in the feature space, thereby significantly enhancing performance.
In contrast, the ALICE and Orco methods each utilize specific losses (angular penalty loss and center orthogonality loss, respectively)
to optimize classification decision boundaries. 
Consequently, integrating ADBS does not yield significant improvements in base classification performance. 
Nevertheless, during incremental sessions, ADBS continues to effectively balance the decision spaces between old and new classes without compromising the classification performance of the combined methods, thereby enhancing the overall model performance.

\subsection{Ablation Study}

We perform ablation studies to investigate the effectiveness of
the key components of the proposed ADBS on the miniImageNet dataset.
Built on the BPNF framework, we initially employed only Cross-Entropy loss as the baseline. 
Subsequently, we progressively integrated the proposed Adaptive Decision Boundary (ADB) and Inter-class Constraint (IC), observing their impacts on performance.
As shown in Table~\ref{tab:ablation_miniimagenet}, each component contributes to performance improvement.
Notably, the ADB strategy significantly enhances classification performance in all sessions, demonstrating that the adaptive strategy effectively adjusts decision boundaries for better space allocation. Additionally, we incorporated an inter-class constraint to increase the distinguishability between classes, which also improved the compactness of embeddings and facilitated FSCIL.~\footnote{More ablation results of CIFAR100 and CUB200 datasets can be found in the supplementary materials.}

\begin{figure}[t]
    \centering
        \includegraphics[width=\linewidth]{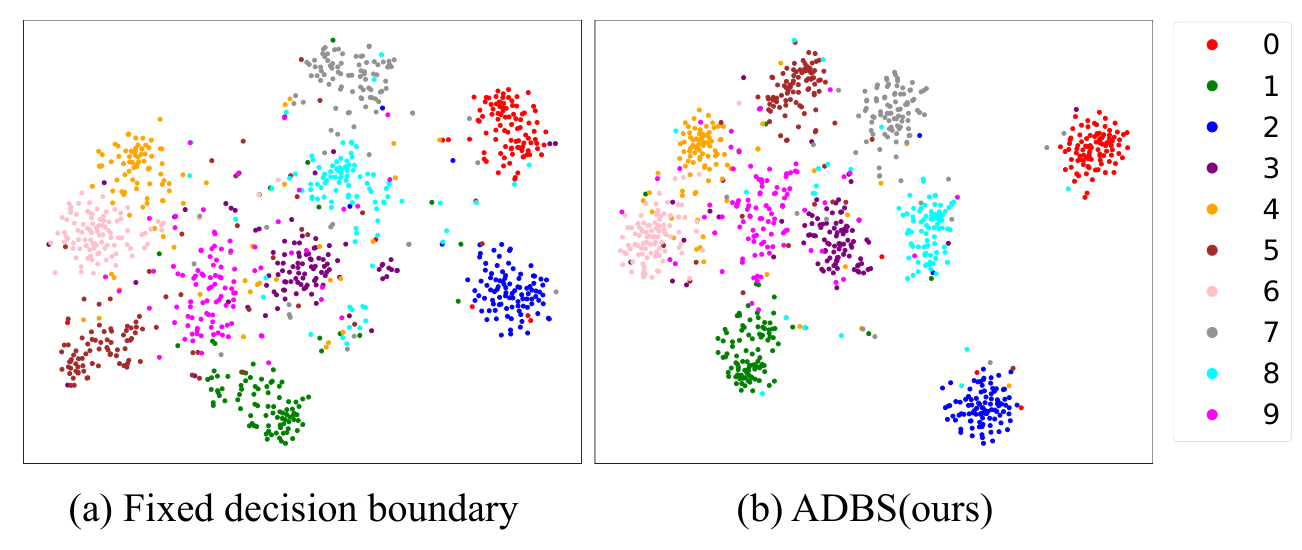}
    \caption{The t-SNE visualization on the miniImageNet dataset of the feature embeddings. Classes 0-5 denote the base classes, whereas classes 6-9 signify the new classes. Our ADBS, as compared to the fixed decision boundary, demonstrates superior class separability.}
    \label{fig:t-sne}
\end{figure}

\begin{figure}[t] 
    \centering
    \includegraphics[width=\linewidth]{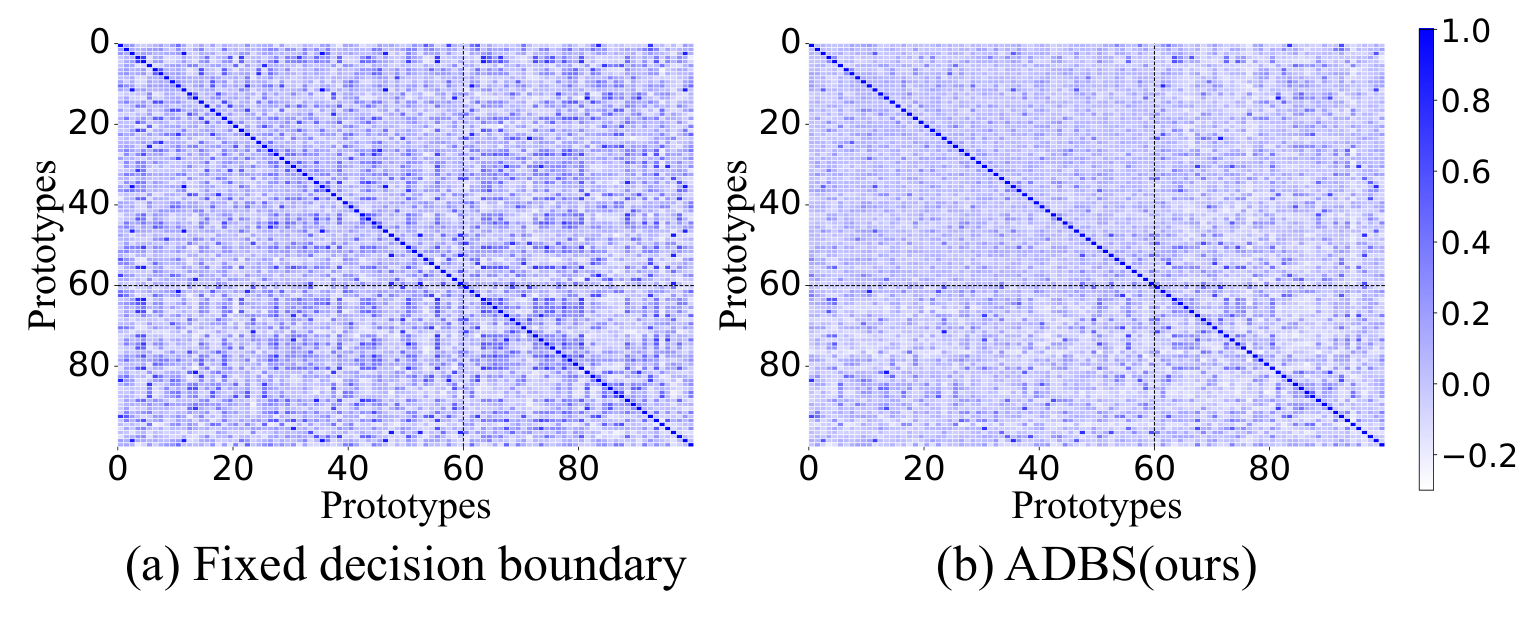}
    \caption{Cosine similarity score between classes on CIFAR100. We use black dashed lines
    to demarcate the regions of base and incremental classes.
    Compared to the fixed decision boundary, our ADBS creates wider spacing between class prototypes, leading to lighter colors in the figure.
    }
    \label{fig:distance}
\end{figure}

\subsection{Further Analysis}

\noindent{\textbf{Visualization of Class Separation}} We employ T-SNE to visualize the feature space of the miniImageNet dataset, as shown in Figure~\ref{fig:t-sne}, with six classes randomly selected from the base classes and four from the new classes.
As shown in Figure~\ref{fig:t-sne} (a), the base classes are densely clustered while the new classes are more scattered, resulting in frequent conflicts within the feature space.
For example, the purple class (class 3) displays conflicts with both the cyan (class 8) and magenta (class 9) new classes.
In contrast, as shown in Figure~\ref{fig:t-sne} (b), our ADBS enhances intra-class cohesion and inter-class separation, clarifying the decision boundaries between classes and effectively mitigating conflicts between old and new classes in the feature space.

\begin{table}[t]
    \centering
    \begin{tabular}{cccccccccc}
        \toprule
         &  
        {CUB200} &  {CIFAR100} & {miniImageNet} \\
        \midrule
        baseline & {0.3517} & {0.8846} & {0.8944}  \\
        ADBS & \textbf{0.3584} & \textbf{0.9497} & \textbf{0.9477}  \\
        \midrule
        $\Delta_{impro.}$ & {+0.0067} & {+0.0651} & {+0.0533}  \\
    \bottomrule
    \end{tabular}
    \caption{
    The degree of class separation on three FSCIL benchmarks. $\Delta_{impro.}$: Relative improvements of the overall degree compared to the fixed decision boundary baseline.
    }
    \label{tab:correlation}
\end{table}

\noindent {\textbf{Analysis of the Class Separation Degree}} We first compute the cosine similarity of two class prototypes (consistent with the classifier) to analyze the similarity of the two classes, shown in Figure~\ref{fig:distance}.
We observe that a fixed decision boundary strategy causes class prototypes to cluster closely, obscuring boundaries and leading to classification ambiguities. 
In contrast, our ADBS significantly separates class prototypes, establishing clearer boundaries between classes for more accurate classification.
To further explore our effectiveness in enhancing class separation, we quantified the degree of separation $D_{cs}$ for each method by calculating the average cosine distance between the prototypes of any two classes as:
\begin{equation} 
    D_{cs} = \frac{1}{N^2} \sum_{i=1}^N \sum_{j=1}^N (1-\text{sim}
    (p_i,p_j)),
\end{equation}
where $N$ indicates the number of all classes.
Table~\ref{tab:correlation} presents the results of the baseline and our ADBS on three FSCIL benchmarks.
The results demonstrate that our ADBS effectively separates different classes, mitigating class boundary conflicts and thereby enhancing performance.
It is worth noting that, compared to the other two datasets, CUB200 primarily focuses on fine-grained classification scenarios, where separating different classes is more challenging (\textit{i.e.}, a lower score). 
However, applying our ADBS still achieves a slight improvement, underscoring its generalizability in complex scenarios.

\section{Conclusion}

The conflict between new and old classes within the feature space presents a notable challenge for FSCIL methods. 
In this paper, we propose a plug-and-play \textit{Adaptive Decision Boundary Strategy} (ADBS) designed specifically to mitigate this issue. 
ADBS comprises two key components, \textit{i.e.}, an \textit{Adaptive Decision Boundary} (ADB) and an \textit{Inter-class Constraint} (IC).
ADB assigns a distinct decision boundary to each class and adaptively adjusts it in each session, while the IC is designed to enhance class distinguishability.
Moreover, ADBS offers plug-and-play functionality, enabling seamless integration into various FSCIL methods.
Extensive experiments on three benchmarks demonstrate that incorporating ADBS consistently enhances the performance of existing FSCIL methods, highlighting the effectiveness of our proposed strategy.

\section*{Acknowledgements}

This research was supported by National Natural Science Foundation of China (NSFC) (62376088), Hebei Natural Science Foundation (F2024202047). This research was supported in part by the Wallenberg-NTU Presidential Postdoctoral Fellowship (Award No. 024560-00001).

\bibliography{aaai25}

\clearpage

\setcounter{figure}{0}
\setcounter{table}{0}
\setcounter{section}{0} 
\setcounter{proposition}{0} 
\setcounter{equation}{0}
\setcounter{proof}{0}

\begin{strip}
\centering
\Large \textbf{Supplementary Material}
\end{strip}
 
\section{Introduction}

In this supplementary material, we provide a detailed proof of Proposition 1 in Section A. 
Section B presents comprehensive information about the three datasets and experimental settings. 
Furthermore, Section C contains more experimental results, as well as the pseudo-code of our method in incremental sessions.

\section{A. Proof of the Proposition 1}
\label{sec:proof}

We first recall the Proposition~\ref{propos1} as follows and then give the proof.
\begin{proposition}
\label{propos1}
Given a classifier with weights $W$, the adaptive boundary strategy could help to better separate classes $i,j$ if the boundary weights $m_i$ and $m_j$ satisfy the following equation:
\begin{equation}
    \quad (1 - m_i)p_i^\top w_i + (m_j - 1)p_i^\top w_j \leq 0, \quad \forall i \neq j, 
    \label{eq:constraint}
\end{equation}
where $p_i$ denotes the prototype of class $i$. $w_i$ and $w_j$ correspond to the weights of classes $i$ and $j$ in the classifier, respectively.
\end{proposition}

\begin{proof}
For a single image sample \(x\) from class \(i\), under fixed decision boundaries, the probability that \(x\) belongs to class \(i\) is computed as follows:

\begin{equation}
    P_{fdb}(x \in i ) = \frac{e^{ \langle w_i,f(x) \rangle}}{e^{ \langle w_i,f(x) \rangle} + \sum_{i \neq j} e^{ \langle w_j,f(x) \rangle}}, 
\end{equation}
where \(w\) denotes the weight of the classifier, \( f(x) \) represents the feature vector obtained after input \( x \) is processed by the feature extractor,
and \(\langle \cdot, \cdot \rangle\) calculates the cosine similarity, which is less than or equal to 1.

Upon integrating the \textit{Adaptive Decision Boundary Strategy}, we assign a boundary \(m\) for each class, altering the probability for \(x\) belonging to class \(i\) as:

\begin{equation}
\begin{split}
    \,  P_{adb}(x \in i )
    & = \frac{e^{ m_i \langle w_i,f(x) \rangle}}{e^{ m_i \langle w_i,f(x) \rangle} + \sum_{i \neq j} e^{ m_j \langle w_j,f(x) \rangle}} \\
    & = \frac{e^{\langle w_i,f(x) \rangle}}{e^{\langle w_i,f(x) \rangle} + \sum_{i \neq j} e^{ m_j \langle w_j,f(x) \rangle + (1-m_i){\langle w_i,f(x) \rangle}}}. \\
\end{split}
\label{eq:combination}
\end{equation}

To maintain the distinguishability among classes following the addition of boundaries, we enforce these restrictive constraints on \(m\):

\begin{equation}
    \begin{split}
        & \, P_{adb}(x \in i ) \geq P_{fdb}(x \in i ) \\
        & \Longleftrightarrow
        \frac{e^{\langle w_i,f(x) \rangle}}{e^{\langle w_i,f(x) \rangle} + \sum_{i \neq j} e^{ m_j \langle w_j,f(x) \rangle + (1-m_i){\langle w_i,f(x) \rangle}}}  \\
        &\quad\quad \geq \frac{e^{ \langle w_i,f(x) \rangle}}{e^{ \langle w_i,f(x) \rangle} + \sum_{i \neq j} e^{ \langle w_j,f(x) \rangle}} \\
        & \Longleftrightarrow \sum_{i \neq j} (e^{ m_j \langle w_j,f(x) \rangle + (1-m_i){\langle w_i,f(x) \rangle}} - e^{ \langle w_j,f(x) \rangle}) \leq 0 \\
    \end{split}
\end{equation}

Hence, our condition holds as long as each term in the summation remains non-positive:

\begin{equation}
    \begin{split}
        & \, e^{ m_j \langle w_j,f(x) \rangle + (1-m_i){\langle w_i,f(x) \rangle}} - e^{ \langle w_j,f(x) \rangle} \leq 0 \\
        & \Longleftrightarrow m_j \langle w_j,f(x) \rangle + (1-m_i){\langle w_i,f(x) \rangle} - { \langle w_j,f(x) \rangle} \leq 0.       
    \end{split}
\end{equation}

We hypothesize that after model training, the feature vector \(f(x)\) closely aligns with its class prototype, enabling the prototype to approximately represent the sample’s features( \textit{i.e.}, \(f(x) \approx p_c\)). 
Consequently, this alignment facilitates:
\begin{equation}
    \begin{split}
        & \, m_j \langle w_j,p_i \rangle + (1-m_i){\langle w_i,p_i \rangle} - { \langle w_j,p_i \rangle} \leq 0 \\
        & \Longleftrightarrow (1 - m_i)p_i^\top w_i + (m_j - 1)p_i^\top w_j \leq 0.
    \end{split}
\end{equation}

Thus, our inequality is universally valid across all classes, establishing our \textit{Inter-class Constraint} as derived:
\begin{equation}
    \begin{split}
        \quad (1 - m_i)p_i^\top w_i + (m_j - 1)p_i^\top w_j \leq 0, \quad \forall i \neq j.
    \end{split}
\end{equation}
\end{proof}

\begin{table*}[ht]
    \centering
    \setlength{\tabcolsep}{1mm}
    {
    \begin{tabular}{lcccccccccccccc}
        \toprule
        \multirow{2}{*}{Methods} & \multicolumn{11}{c}{Accuracy in each session (\%)} & \multirow{2}{*}{Average} 
        \\
        \cmidrule(lr){2-12} 
        & 0 & 1 & 2 & 3 & 4 & 5 & 6 & 7 & 8 & 9 & 10 & Acc. 
        \\
        \midrule
Topic & {68.68} & {62.49} & {54.81} & {49.99} & {45.25} & {41.40} & {38.35} & {35.36} & {32.22} & {28.31} & {26.28} & {43.92 }  \\
CEC & {75.85} & {71.94} & {68.50} & {63.50} & {62.43} & {58.27} & {57.73} & {55.81} & {54.83} & {53.52} & {52.28} & {61.33 }  \\
FACT & {75.90} & {73.23} & {70.84} & {66.13} & {65.56} & {62.15} & {61.74} & {59.83} & {58.41} & {57.89} & {56.94} & {64.42 }  \\
CLOM† & {79.57} & {76.07} & {72.94} & {69.82} & {67.80} & {65.56} & {63.94} & {62.59} & {60.62} & {60.34} & {59.58} & {67.17 }  \\
DBONet† & {78.66} & {75.53} & {72.72} & {69.45} & {67.21} & {65.15} & {63.03} & {61.77} & {59.77} & {59.01} & {57.42} & {66.34 }  \\      
MCNet & {77.57} & {73.96} & {70.47} & {65.81} & {66.16} & {63.81} & {62.09} & {61.82} & {60.41} & {60.09} & {59.08} & {65.57 }  \\       
SoftNet & {78.11} & {74.51} & {71.14} & {62.27} & {65.14} & {62.27} & {60.77} & {59.03} & {57.13} & {56.77} & {56.28} & {63.95 }  \\
WaRP  & {77.74} & {74.15} & {70.82} & {66.90} & {65.01} & {62.64} & {61.40} & {59.86} & {57.95} & {57.77} & {57.01} & {64.66 }  \\
TEEN  & {77.26} & {76.13} & {72.81} & {68.16} & {67.77} & {64.40} & {63.25} & {62.29} & {61.19} & {60.32} & {59.31} & {66.63 }  \\
NC-FSCIL & {80.45} & {75.98} & {72.30} & {70.28} & {68.17} & {65.16} & {64.43} & {63.25} & {60.66} & {60.01} & {59.44} & {67.28 }  \\      
ALFSCIL & {79.79} & {76.53} & {73.12} & {69.02} & {67.62} & {64.76} & {63.45} & {62.32} & {60.83} & {60.21} & {59.30} & {67.00 }  \\     
DyCR & {77.50} & {74.73} & {71.69} & {67.01} & {66.59} & {63.43} & {62.66} & {61.69} & {60.57} & {59.69} & {58.46} & {65.82 }  \\
\midrule
baseline & {74.37} & {69.98} & {67.11} & {63.02} & {63.04} & {59.97} & {59.80} & {58.25} & {57.05} & {56.85} & {56.26} & {62.34 }  \\
baseline+ADBS & \textbf{79.99} & \textbf{75.89} & \textbf{72.53} & \textbf{68.33} & \textbf{67.92} & \textbf{64.75} & \textbf{64.10} & \textbf{62.93} & \textbf{61.31} & \textbf{60.88} & \textbf{59.65} & \textbf{67.12 }  \\
  & \textbf{(+5.62)} & \textbf{(+5.91)} & \textbf{(+5.43)} & \textbf{(+5.31)} & \textbf{(+4.88)} & \textbf{(+4.78)} & \textbf{(+4.30)} & \textbf{(+4.69)} & \textbf{(+4.26)} & \textbf{(+4.03)} & \textbf{(+3.39)} & \textbf{(+4.78)  } \\
\midrule
OrCo* & {74.58} & {65.99} & {64.72} & {63.06} & {61.79} & {59.55} & {59.21} & {58.46} & {56.97} & {57.99} & {57.32} & {61.79 }  \\
OrCo+ADBS & \textbf{77.37} & \textbf{71.68} & \textbf{69.62} & \textbf{66.66} & \textbf{64.87} & \textbf{62.71} & \textbf{62.14} & \textbf{61.33} & \textbf{59.78} 
& \textbf{60.01} & \textbf{58.89} & \textbf{65.01 }  \\
  & \textbf{(+2.79)} & \textbf{(+5.70)} & \textbf{(+4.90)} & \textbf{(+3.59)} & \textbf{(+3.08)} & \textbf{(+3.17)} & \textbf{(+2.93)} & \textbf{(+2.87)} & \textbf{(+2.80)} & \textbf{(+2.02)} & \textbf{(+1.57)} & \textbf{(+3.22)  } \\
\midrule
ALICE†* & {77.72} & {70.19} & {68.54} & {65.32} & {63.78} & {61.40} & {60.77} & {59.82} & {58.62} & {58.86} & {58.51} & {63.96 }  \\
ALICE+ADBS & \textbf{77.97} & \textbf{70.31} & {68.54} & \textbf{65.37} & \textbf{63.85} & \textbf{61.53} & \textbf{60.92} & \textbf{60.07} & \textbf{58.78} & \textbf{58.97} & \textbf{58.68} & \textbf{64.09 }  \\
  & \textbf{(+0.24)} & \textbf{(+0.12)} & {(+0.00)} & \textbf{(+0.05)} & \textbf{(+0.07)} & \textbf{(+0.13)} & \textbf{(+0.15)} & \textbf{(+0.25)} & \textbf{(+0.16)} & \textbf{(+0.11)} & \textbf{(+0.17)} & \textbf{(+0.13)  } \\
\midrule
SAVC* & {81.68} & {77.61} & {74.84} & {70.02} & {69.36} & {66.61} & {65.82} & {64.80} & {63.14} & {62.55} & {62.00} & {68.95 }  \\       
SAVC+ADBS & \textbf{81.88} & \textbf{77.86} & \textbf{75.15} & \textbf{70.67} & \textbf{70.49} & \textbf{67.67} & \textbf{67.06} & \textbf{66.03} & \textbf{64.50} 
& \textbf{63.97} & \textbf{63.32} & \textbf{69.87 }  \\
  & \textbf{(+0.21)} & \textbf{(+0.25)} & \textbf{(+0.31)} & \textbf{(+0.65)} & \textbf{(+1.13)} & \textbf{(+1.06)} & \textbf{(+1.23)} & \textbf{(+1.24)} & \textbf{(+1.36)} & \textbf{(+1.42)} & \textbf{(+1.32)} & \textbf{(+0.93)  } \\
        \bottomrule
    \end{tabular}
    }
    \caption{Comparison with SOTA methods on CUB200 dataset for FSCIL. {†: Boundary-based method. *: Reproduced results.}}
    \label{tab:cub200}
\end{table*}

\section{B. Datasets and Experimental Settings}
\label{sec:datasets}

\subsection{B.1. Dataset}

Following the protocols established by TOPIC~\cite{tao2020few}, we
evaluate our method across several datasets: CIFAR100~\cite{krizhevsky2009learning}, miniImageNet~\cite{deng2009imagenet}, and CUB200~\cite{wah2011caltech}.
CIFAR100 comprises 100 image classes, each featuring 500 training and 100 test images, with a resolution of 32 × 32 pixels. 
miniImageNet includes 60,000 images distributed among 100 classes, with each class containing 500 training images and 100 test images, at a resolution of 84 × 84 pixels. CUB200, a fine-grained image classification benchmark, consists of 200 bird species, with 5,994 training images and 5,794 test images, each measuring 224 × 224 pixels.

In the FSCIL task, we employ the same experimental setup outlined in previous studies~\cite{tao2020few,song2023learning,peng2022few}. 
For CIFAR100, we utilize 60 base classes and 40 novel classes, with the model adapting to new data in a 5-way 5-shot format over 8 sessions.
Similarly, for miniImageNet, we follow the CIFAR100 configuration, partitioning it into 60 base classes and 40 novel classes. In contrast, CUB200 is segmented into 100 base classes and 100 novel classes, where the model engages in a 10-way 5-shot learning process over 10 sessions.

\begin{table*}[ht]
    \centering
     \begin{tabular}{cccccccccccccc}
        \toprule
        \multirow{2}{*}{\textbf{ADB}} & \multirow{2}{*}{\textbf{IC}} & \multicolumn{9}{c}{Accuracy in each session (\%)} & \multirow{2}{*}{$\Delta_{last}$} 
        \\
        \cmidrule(lr){3-11} 
        & & 0 & 1 & 2 & 3 & 4 & 5 & 6 & 7 & 8 &  \\
        \midrule
                    &            & {73.92} & {67.14} & {63.71} & {60.07} & {57.10} & {54.85} & {52.52} & {50.49} & {48.60} & {-}  \\
        \checkmark  &            & {79.97} & {74.00} & {69.53} & {64.87} & {61.58} & \textbf{{58.46}} & {55.70} & {53.39} & {51.09} & {+2.49}  \\
        \checkmark  & \checkmark & \textbf{{80.32}} & \textbf{{75.22}} & \textbf{{71.11}} & \textbf{{65.99}} & \textbf{{62.46}} & {58.38} & \textbf{{55.96}} & \textbf{{53.72}} & \textbf{{51.15}} & \textbf{{+2.55}}  \\
        \bottomrule
    \end{tabular}
    \caption{Ablation studies on CIFAR100 benchmark. \textbf{ADB} and \textbf{IC} denote \textit{Adaptive Decision Boundary} and \textit{Inter-class Constraint}, respectively.
    \(\Delta_{last}\): Relative improvements of the last sessions compared to the fixed decision boundary baseline.}
    \label{tab:ablation_cifar100}
\end{table*}

\begin{table*}[ht]
    \centering
     \begin{tabular}{cccccccccccccc}
        \toprule
        \multirow{2}{*}{\textbf{ADB}} & \multirow{2}{*}{\textbf{IC}} & \multicolumn{11}{c}{Accuracy in each session (\%)} & \multirow{2}{*}{$\Delta_{last}$} 
        \\
        \cmidrule(lr){3-13} 
        & & 0 & 1 & 2 & 3 & 4 & 5 & 6 & 7 & 8 & 9 & 10 &   \\
        \midrule
                &            & {74.37} & {69.98} & {67.11} & {63.02} & {63.04} & {59.97} & {59.80} & {58.25} & {57.05} & {56.85} & {56.26} & {-}  \\
    \checkmark  &            & \textbf{{79.99}} & {75.08} & {71.60} & {67.45} & {67.30} & {64.18} & {63.24} & {62.03} & {60.45} & {59.88} & {58.94} & {+2.67}  \\
    \checkmark  & \checkmark & \textbf{{79.99}} & \textbf{{75.89}} & \textbf{{72.53}} & \textbf{{68.33}} & \textbf{{67.92}} & \textbf{{64.75}} & \textbf{{64.10}} & \textbf{{62.93}} & \textbf{{61.31}} & \textbf{{60.88}} & \textbf{{59.65}} & \textbf{{+3.39}}  \\
        \bottomrule
    \end{tabular}
    \caption{Ablation studies on CUB200 benchmark. \textbf{ADB} and \textbf{IC} denote \textit{Adaptive Decision Boundary} and \textit{Inter-class Constraint}, respectively.
    \(\Delta_{last}\): Relative improvements of the last sessions compared to the fixed decision boundary baseline.}
    \label{tab:ablation_cub200}
\end{table*}

\subsection{B.2. Implementation Details}
We implement our method using PyTorch, and all experiments are
conducted in the following environment:  Intel(R) Xeon(R) Gold 6342 @2.80GHz CPU, NVIDIA A30 Tensor Core GPU, and Ubuntu 20.04.6 LTS operation system.
Under the above environment settings, we integrate our approach into four methods: the standard FSCIL baseline, SAVC~\cite{song2023learning}, ALICE~\cite{peng2022few}, and OrCo~\cite{ahmed2024orco}.

\noindent\textbf{Standard FSCIL Baseline}: Following~\cite{hersche2022constrained,yang2023neural}, we employ ResNet-18~\cite{he2016deep} as the backbone for CUB200, and ResNet-12 for miniImageNet and CIFAR100.
Standard preprocessing and augmentation techniques such as random resizing, flipping, and color jittering are applied~\cite{tao2020few,zhang2021few}.
We use the SGD optimizer with Nesterov momentum (0.9).
All models have uniform batch sizes of 128 and undergo 120 training epochs for CUB200 and miniImageNet, and 100 epochs for CIFAR100.
The initial learning rates are set at 0.1 for CIFAR100 and miniImageNet, and 0.002 for CUB200, with boundary fine-tuning conducted for up to 10 epochs in each incremental session to prevent overfitting.

\noindent\textbf{SAVC}: 
We replace the backbone with ResNet-12 for CIFAR100 and miniImageNet and simplify the augmentation on miniImageNet from 12 to 2-fold rotation to streamline the training process, as detailed in~\cite{song2023learning}. 
Incremental sessions include up to 10 epochs of boundary fine-tuning.
All other settings are consistent with SAVC.

\noindent\textbf{ALICE}: Boundaries are fine-tuned over up to 10 epochs in incremental sessions. 
We strictly follow the original experimental settings detailed in~\cite{peng2022few} for all datasets, with all other parameters following the ALICE baseline unless specified.

\noindent\textbf{OrCo}: 
We initiate the base session with a pretrained model provided by the original authors. 
For the CUB200 dataset, both the feature extractor and the projection layer are trained for 200 epochs
Conversely, in CIFAR100 and miniImageNet, only the projection layer undergoes training—10 epochs for CIFAR100 and 100 epochs for miniImageNet.
Furthermore, we fine-tune the boundaries for an additional 10 epochs at the end of each session. 
Unless explicitly specified otherwise, we adhere to all of OrCo’s standard settings.

\begin{figure}[!htbp]
    \centering
    \includegraphics[width=\linewidth]{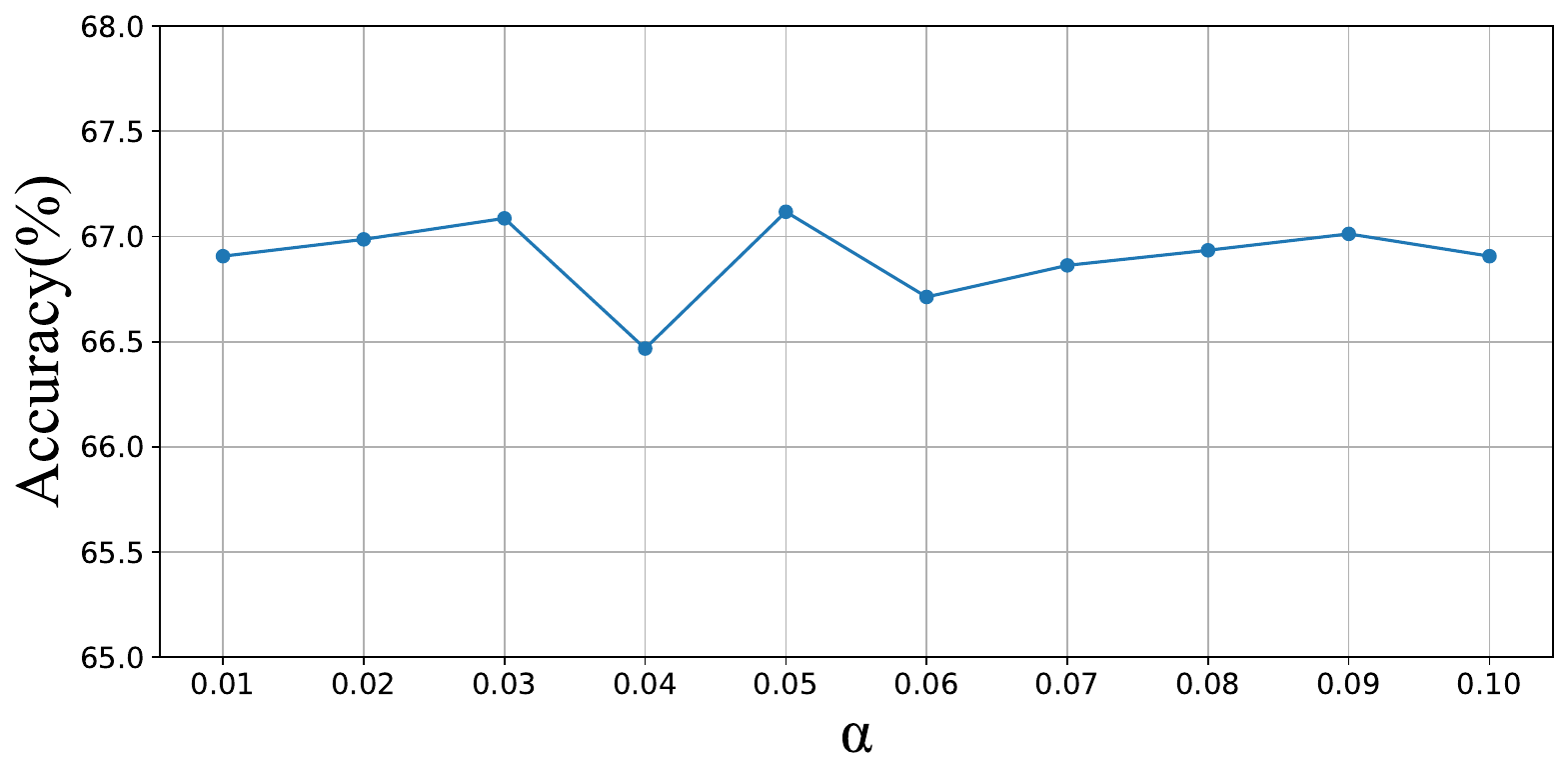}
    \caption{Influence of hyper-parameter $alpha$ on average accuracy of our ADBS integrated with the FSCIL baseline on CUB200 dataset.}
    \label{fig:hyperparas}
\end{figure}

\begin{figure}[!htbp]
    \centering
    \includegraphics[width=\linewidth]{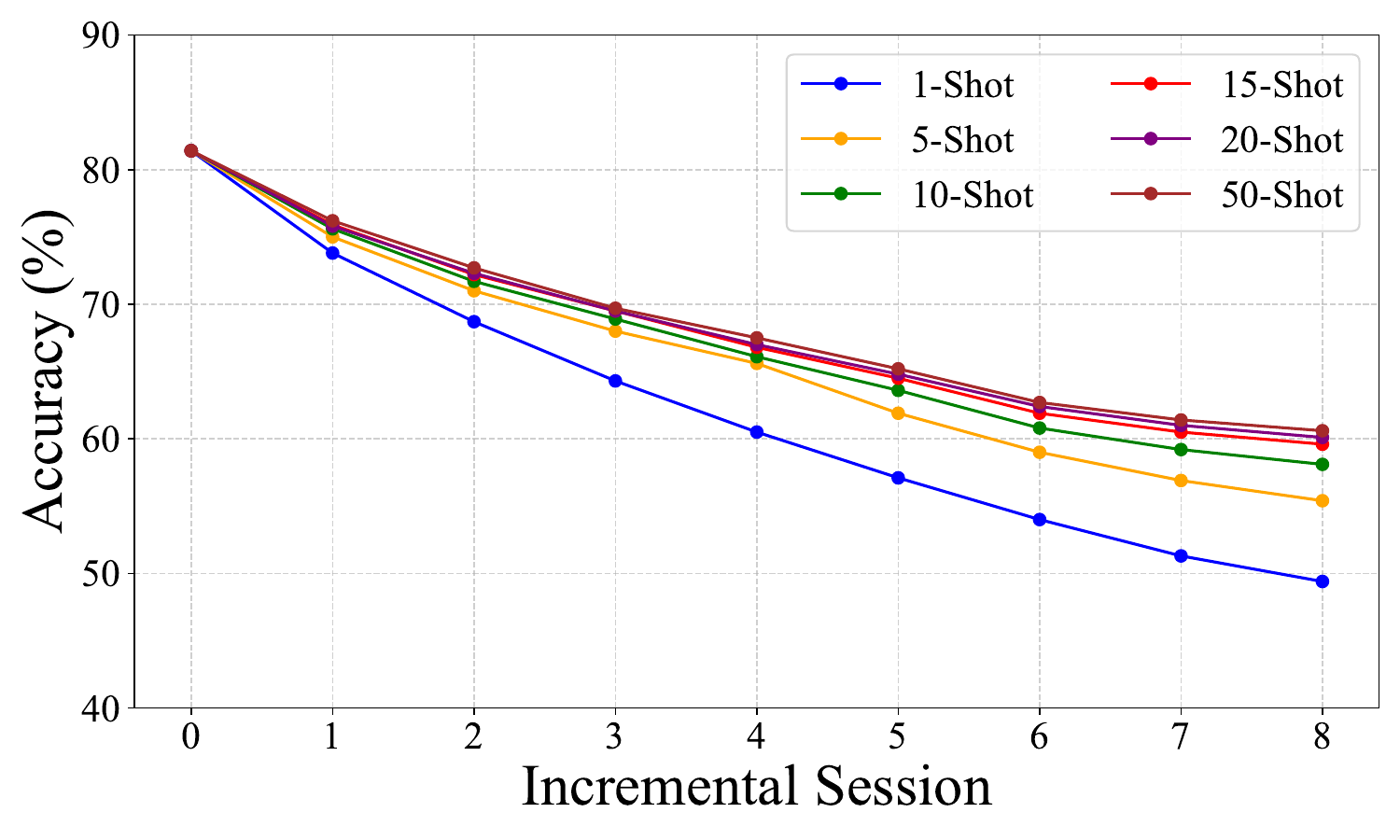}
    \caption{Influence of incremental shots on each session's accuracy for 5-Way FSCIL task on miniImageNet dataset.}
    \label{fig:5-way-k-shot}
\end{figure}

\begin{figure}[!htbp]
    \centering
    \includegraphics[width=\linewidth]{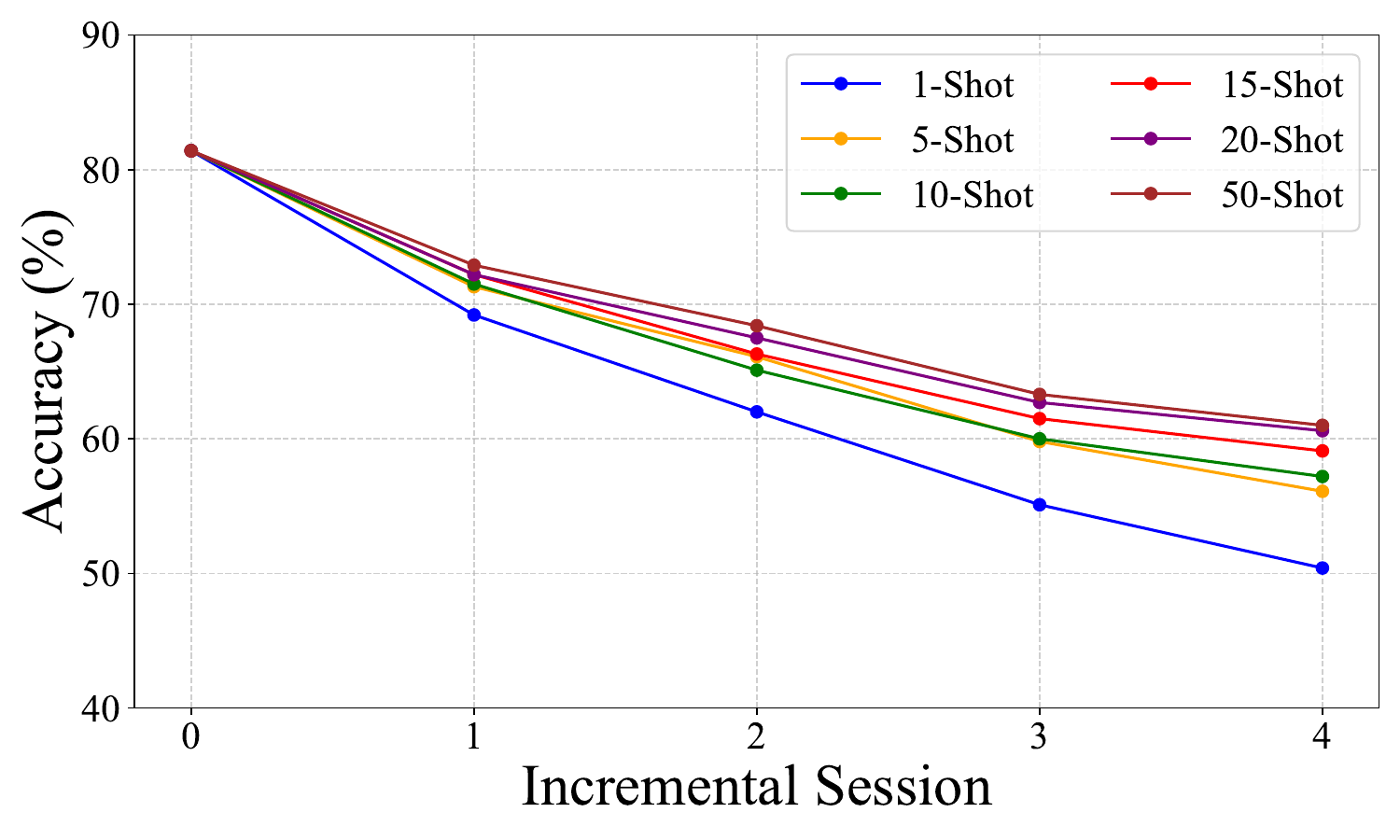}
    \caption{Influence of incremental shots on each session's accuracy for 10-Way FSCIL task on miniImageNet dataset.}
    \label{fig:10-way-k-shot}
\end{figure}

\section{C. More Experimental Results}
\label{sec:results}

\textbf{Main results on CUB200.} We conduct experimental comparison on CUB200 dataset, as shown in Table~\ref{tab:cub200}.
The results demonstrate that our proposed ADBS consistently improves the performance of all integrated FSCIL methods. 
Notably, integrating SAVC with ADBS (SAVC+ADBS) achieves state-of-the-art performance.

\noindent\textbf{More results on ablation study.} We conduct more ablation studies to show the effectiveness of the key components of the proposed ADBS on the CIFRA100 and CUB200 datasets. As shown in Tables~\ref{tab:ablation_cifar100} and~\ref{tab:ablation_cub200}, both ADB and IC improve the classification performance in most sessions.

\noindent\textbf{Analysis of hyper-parameter $\alpha$.}
Our loss function includes a single hyper-parameter, the trade-off parameter $\alpha$, which balances the two terms in the total loss function.
We evaluated the average accuracy of our proposed ADBS when integrated with the FSCIL baseline on the CUB200 dataset by varying $\alpha$ from $0.01$ to $0.1$, as illustrated in Figure~\ref{fig:hyperparas}.
The results indicate that our method is robust to variations in $\alpha$, achieving optimal performance at $\alpha = 0.05$.
Thus, we use the $\alpha = 0.05$ for all experiments in our main experiment.

\noindent\textbf{Experiments on More Practical Settings.}
The standard FSCIL datasets assume a 5-shot setting in incremental sessions.
To explore more practical scenarios, we conducted additional experiments on the miniImageNet dataset, varying the number of classes (5 and 10) and the shot number ( \(K \in \{1,5,10,15,20,50\}\) ).
As shown in Figures~\ref{fig:5-way-k-shot} and~\ref{fig:10-way-k-shot}, the results indicate that the number of classes has minimal impact on performance.
In contrast, increasing the sample size improves accuracy due to enhanced prototype estimation and boundary optimization.
Notably, our method achieves robust performance even with limited samples.

\noindent\textbf{Pseudo-code of ADBS.} The pseudo-code for the incremental training process of our ADBS method is presented in Algorithm~\ref{alg:incremental_training}.

\begin{algorithm}[tb]
\caption{Incremental Training of ADBS}
\label{alg:incremental_training}
\textbf{Input}: Incremental training dataset for session $t$: \(D^t_{\text{train}}\), Feature extractor: \(f\), Classifier weights:  
\(W\), 
Previous decision boundaries:  
\(M\), 
Hyperparameter: \(\alpha\) \\
\textbf{Output}: Updated classifier weights:  
\(W\), 
Updated decision boundaries:  
\(M\)
\begin{algorithmic}[1] 
\STATE Update the classifier by incorporating prototypes of the new classes.
\STATE Expand \(M\) to include decision boundaries for the new classes:
\[
M = \bigcup_{s=0}^{t} \{ m_1^s, m_2^s, \dots, m_{|\mathcal{C}^s|}^s \}.
\]
\STATE Initialize decision boundaries for the new classes based on the mean of all previous class boundaries:
\[
m_i^t = \frac{1}{\sum_{s=0}^{t-1} |\mathcal{C}^s|} \sum_{s=0}^{t-1} \sum_{j=1}^{|\mathcal{C}^s|} m_j^s, \quad 1 \leq i \leq K.
\]
\REPEAT
    \STATE Sample a mini-batch \(B = \{ (x_i, y_i) \}_{i=1}^n \) from \(D^t_{\text{train}}\).
    \STATE Compute the classification loss \(\mathcal{L}_{cls}\).
    \STATE Compute the inter-class constraint loss \(\mathcal{L}_{IC}\).
    \STATE Compute the total loss: \(\mathcal{L} = \mathcal{L}_{cls} + \alpha \mathcal{L}_{IC}\).
    \STATE Update the decision boundaries for the new classes: \(m_i^t, \quad 1 \leq i \leq K\), using gradient descent.
\UNTIL{the predefined number of epochs is reached.}
\end{algorithmic}
\end{algorithm}

\end{document}